\documentclass[twoside]{article}

\usepackage[accepted]{aistats2020}

\usepackage[round]{natbib}

\usepackage[utf8]{inputenc} % allow utf-8 input
\usepackage[T1]{fontenc}    % use 8-bit T1 fonts
\usepackage{hyperref}       % hyperlinks
\usepackage{url}  	          % simple URL typesetting
\usepackage{booktabs}       % professional-quality tables
\usepackage{amsfonts}       % blackboard math symbols
\usepackage{nicefrac}       % compact symbols for 1/2, etc.
\usepackage{microtype}      % microtypography
\usepackage[english]{babel}
\usepackage{amsmath}
\usepackage{amsthm}
\usepackage{amssymb}
\usepackage[usenames,dvipsnames]{xcolor}
\usepackage{graphicx}
\usepackage{tikz}
\usetikzlibrary{shapes.multipart}
\usepackage[colorinlistoftodos, color=orange!50]{todonotes}
\usepackage{fancybox}
\usepackage{epsfig}
\usepackage{soul}
\usepackage[framemethod=tikz]{mdframed}
\usepackage{enumitem}
\usepackage{subcaption}
\usepackage{algorithm}
\usepackage{algorithmic}
\usepackage{dsfont}
\usepackage{appendix}
\usepackage{capt-of}
\usepackage{chngpage}
\usepackage{mathtools}

\newcommand{\ie}{i.e.,\ }
\newcommand{\eg}{e.g.,\ }
\newcommand{\NameOne}[1][]{Privacy Regularized Policy#1\ }

\newcommand\indep{\protect\mathpalette{\protect\independenT}{\perp}}
\def\independenT#1#2{\mathrel{\rlap{$#1#2$}\mkern2mu{#1#2}}}

\newtheorem{thm}{Theorem}
\newtheorem{definition}{Definition}

\newtheorem{lemma}{Lemma}
\newtheorem{coro}{Corollary}

\DeclareMathOperator*{\esssup}{ess\,sup}

\DeclareMathOperator*{\argmin}{arg\,min}

\begin{document}
\twocolumn[

\aistatstitle{Utility/Privacy Trade-off through the lens of Optimal Transport}

\aistatsauthor{Etienne Boursier \And Vianney Perchet}

\aistatsaddress{\parbox{8cm}{\scriptsize \centering Université Paris-Saclay, ENS Paris-Saclay \\ CNRS, Centre Borelli, Cachan, France} \\ \scriptsize \texttt{eboursie@ens-paris-saclay.fr} \And
	\parbox{8cm}{\scriptsize \centering CREST, ENSAE Paris, Palaiseau, France \\ Criteo AI Lab, Paris, France} \\\scriptsize \texttt{vianney.perchet@normalesup.org}} 
]

\begin{abstract}
%Private Learning aims at using private information in order to receive the best possible payoff, without revealing this information. While classical frameworks avoid any adversary to retrieve specific information in databases with sensitive data, we introduce the new problem of \textit{\NameOne[,]}where this information has some finite value. Motivated by strategic applications as auctions for online advertisement, our model leads to satisfying solutions which adapt to the problem features. Any information is leaked only if it leads to a significant increase in the reward. For specific privacy costs, this problem presents satisfying properties of convexity and finiteness of the solution. Also, particular cases of this problem are related to theories of  or Difference of Convex functions (DC) programming and it can then benefit from their recent advances. 
Strategic information is valuable either by remaining private (for instance if it is sensitive) or, on the other hand, by being used publicly to increase some utility. These two objectives are antagonistic and leaking this information might be more rewarding than concealing it. Unlike classical solutions that focus on the first point, we consider instead agents that optimize a natural trade-off between both  objectives.
We formalize this as an optimization problem where the objective mapping is regularized by the amount of information revealed to the adversary (measured as a divergence between the prior and posterior on the private knowledge). Quite surprisingly, when combined with the entropic regularization, the Sinkhorn loss naturally emerges in the optimization objective, making it efficiently solvable. We apply these techniques to preserve some privacy in online repeated auctions.

\end{abstract} 

%\begin{keywords}
% Privacy Learning, Optimal Transport, Non-Convex Optimization
%\end{keywords}
\section{Introduction}

In many economic mechanisms and strategic games involving different agents, asymmetries of information (induced by a private type, some knowledge on the hidden state of Nature, etc.) can and should be leveraged to  increase one's utility. When these interactions between agents are repeated over time, preserving some asymmetry (i.e., not revealing private information) can be  crucial to guarantee a larger utility in the long run. Indeed, the small short term utility of publicly using information can be overwhelmed by the long term effect of revealing it \citep{aumann1995}.

Informally speaking, an agent should use, and potentially reveal,  some private information only if she gets a subsequent utility increase in return. Keeping this information private is no longer a constraint \citep[as in other classical privacy concepts such as differential privacy][]{dwork2006} but becomes part of the objective, which is then to decide how and when to use it. For instance, it might happen that revealing everything is optimal or, on the contrary, that a non-revealing policy is the best one.
This is roughly similar to a poker player deciding whether to bluff or not. In some situations, it might be interesting to focus solely on the utility even if it implies losing the whole knowledge advantage, while in other situations, the immediate profit for using this advantage is so small that playing independently of it (or bluffing) is better.

After a rigorous mathematical formulation of this utility vs.\ privacy trade-off, it appears that this problem can be recast as a regularized optimal transport minimization. In the specific case of entropic regularization, this problem has received a lot of interest in the recent years as it induces a computationally tractable way to approximate an optimal transport distance between distributions and has thus been used in many applications \citep{Cuturi2013}. Our work showcases how the new \NameOne problem benefits in practice from this theory.
\vspace{-0.5em}
\paragraph{Private Mechanisms.}
Differential privacy is the most widely used private learning framework \citep{dwork2011, dwork2006, reed2010} and ensures that the output of an algorithm does not significantly depend on a single element of the whole dataset. These privacy constraints are often too strong for economic applications (as illustrated before, it is sometimes optimal to disclose publicly some private information). $f$-divergence privacy costs have thus been proposed in recent literature as a promising alternative \citep{chaudhuri2019}. These $f$-divergences, such as Kullback-Leibler, are also used by economists to  measure the cost of information from a Bayesian perspective, as in the rational inattention literature \citep{sims2003, matvejka2015, mackowiak2015}. It was only recently that  this approach has been considered to measure  ``privacy losses'' in economic mechanisms \citep{Mu2019}. This model assumes that the designer of the mechanism  has some prior belief on the unobserved and private information. After observing the action of the player, this belief is updated and the cost of information corresponds to the KL between the prior and posterior distributions of this private information.

%It prevents any private value of an individual to be inferred from a public aggregate information. It can be achieved  by adding random noise to either the input, the output or the gradient during the learning process \cite{chaudhuri2011, duchi2013, jayaraman2019}. Its popularity is due to its simple, clear mathematic formulation yet with strong privacy guarantees \citep{broadening}. However, its constraints are too restrictive for many problems and can lead to a low average utility \citep{jayaraman2019}.
%Some relaxations of the differential privacy have been proposed to overcome this point \citep{Smith2009, BunS16, mironov2017, jayaraman2019}. In particular, it is possible to consider average privacy leakage instead of ($\delta-$)worst case \citep{dwork2016, BunS16}, or to consider only inference from some (known) prior instead of all adjacent databases \citep{Smith2009, Issa2016}. %Our framework combines these two relaxations besides adding the privacy cost as a regularization term to the objective. It can however be extended to more complex privacy guarantees.

%The objective of private learning is to limit adversarial inference attacks. Different types of attack exist. Here, we aim at preventing attribute inference attacks \citep{fredrikson2014, yeom2018, mironov2017}, whose goal is to infer the value of private features given public observation. We thus compare the posterior and prior distributions of the private information, as done in the rational inattention literature \cite{sims2003, matvejka2015, mackowiak2015, Mu2019}.

Optimal privacy preserving strategies with privacy constraints have been recently studied in this setting under specific conditions \citep{Mu2019}. Loss of privacy can however be directly considered as a cost in the overall objective and an optimal strategy  reveals information only if it actually leads to a significant increase in utility, whereas constrained strategies systematically reveal as much as allowed by the constraints, without incorporating the additional cost of this revelation.
%Also, previous works in private learning often lead to uniformly add noise to some data, while the noise in our model will be neatly chosen to provide the best possible utility.
\vspace{-0.5em}
\paragraph{Optimal Transport.}
Finding an appropriate way to compare probability distributions  is a major challenge in learning theory. Optimal Transport manages to provide powerful tools to compare distributions in metric spaces \citep{villani2008}. As a consequence, it has received an increasing interest these past years \citep{santambrogio2015}, especially for generative models \citep{arjovsky17, genevay18a, salimans2018}. However, such powerful distances often come at the expense of heavy and intractable computations, which might  not be suitable to learning algorithms. It was recently showcased that adding an entropic regularization term enables fast computations of approximated distances using Sinkhorn algorithm \citep{Sinkhorn67, Cuturi2013}. Since then, the Sinkhorn loss has also shown promising results for applications such as generative models \citep{genevay2016, genevay18a}, domain adaptation \citep{Courty14} and supervised learning \citep{Frogner15}, besides having nice theoretical properties \citep{OTbook, Feydy2018, genevay2018b}.
\vspace{-0.5em}
\paragraph{Contributions and Organization of the paper.}
The new framework of \NameOne is motivated by several applications, presented in Section~\ref{SE:Applications} and is formalized in Section~\ref{SE:model}. This problem is mathematically formulated as some optimization problem (yet eventually in an infinite dimensional space), which is convex if the privacy cost is an $f$-divergence, see Section \ref{sec:convex}. Also, if the private information space is discrete, this problem admits an optimal discrete distribution. The minimization problem then becomes dimensionally finite, but non-convex.

If  the Kullback-Leibler divergence between the prior and the posterior is  considered for the cost of information, the problem becomes a Sinkhorn loss minimization problem.  Optimal transport techniques are developed in  Section \ref{sec:transport}  (based on recent machinery) to compute partially revealing policies. Finally, with a linear utility cost, the problem is equivalent to the minimization of the difference of two convex functions. Using the theories of these specific problems, different optimization methods can be compared, which illustrates the practical aspect of our new model. This is done in Section~\ref{sec:expe} where we also compute partially revealing strategies for repeated auctions.
\vspace{-0.5em}
\section{Some Applications}\label{SE:Applications}
\vspace{-0.5em}
Our model is motivated by different applications described in this section: online repeated auctions and learning models on external servers.
\vspace{-0.5em}
\subsection{Online repeated auctions}

When a website wants to sell an advertisement slot, firms such as Google or Criteo take part in an auction to buy this slot for one of their customer, a process illustrated in Figure~\ref{fig:advertisement}. As this interaction happens each time a user lands on the website, this is no longer a one-time auction problem, but repeated auctions where the seller and/or the competitor might observe not just one bid, but a distribution of bids. As a consequence, if a firm were bidding truthfully, seller and other bidders would have access to its true value distribution $\mu$. This has two possible downsides.

First, if the value distribution $\mu$ was known to the auctioneer, she could maximize her revenue at the expense of the bidder utility \citep{amin2013, amin2014, feldman2016, golrezaei2018}, for instance with personalized reserve prices. Second, the auctioneer can sometimes take part in the auction and becomes a direct concurrent of the bidder (this might be a unique characteristic of online repeated auctions for ads). For instance, Google is both running some auction platforms and bidding on some ad slots for their client. As a consequence, if the distribution $\mu$ was perfectly known to some concurrent bidder, he could use it in the future, by bidding more or less aggressively or by trying to conquer new markets.

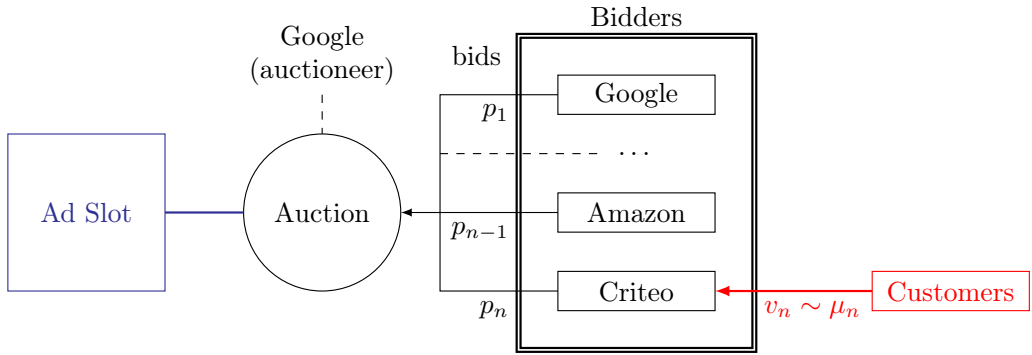
\begin{figure*}[t]
\center
\resizebox{0.8\linewidth}{!}{
\begin{tikzpicture}[every text node part/.style={align=center}]
% Auction/Ad Slot and auctioneer part
\draw (5.5,0) circle (1) ;
\draw (5.5,0) node{Auction} ;
\draw [Blue, thick] (3.5,0) -- (4.5,0) ;
\draw [Blue] (1.5, -1) rectangle (3.5, 1) ;
\draw [Blue] (2.5, 0) node{Ad Slot} ;
\draw (5.5, 2) node{Google\\ (auctioneer)} ;
\draw [dashed] (5.5, 1.5) -- (5.5, 1) ;

% bidder rectangles
\draw [thick, double] (8, -1.75) rectangle (11, 2.25) ;
\draw (8.5, -1.25) rectangle (10.5, -0.75) ;
\draw (8.5, -0.25) rectangle (10.5, 0.25) ;
\draw (8.5, 1.75) rectangle (10.5, 1.25) ;
% bidder names
\draw (9.5, -1) node{Criteo} ;
\draw (9.5, 1.5) node{Google} ;
\draw (9.5, 0) node{Amazon} ;
\draw (9.5, 0.75) node{\ldots} ;
\draw (9.5, 2.25) node[above]{Bidders} ;

%\draw [>=latex, <-, double, thick] (7, 0) -- (8, 0) ;
\draw [dashed] (9, 0.75) -- (7, 0.75) ;
\draw (7.9, 1.75) node[above left]{bids} ;
\draw (8.5, 1.5) -- (7, 1.5) -- (7, 0) ;
\draw (8, 1.5) node[below left]{$p_1$} ;
\draw (8.5, -1) -- (7, -1) -- (7, 0) ; %Criteo arrow
\draw (8, -1) node[below left]{$p_n$} ;
\draw (8.5, 0) -- (7, 0) ;
\draw (8, 0) node[below left]{$p_{n-1}$} ;
\draw [>=latex, ->] (7, 0) -- (6.5, 0) ;
%\draw [double, dashed] (9.5, -1.75) -- (9.5, -2) -- (6.25, -2) ; % old link auctioneer -- bidder

\draw [red] (12.5, -1.25) rectangle (14.5, -0.75) ;
\draw [red] (13.5, -1) node{Customers} ;
\draw [red, thick, >=latex, ->] (12.5, -1) -- (10.5, -1) ;
\draw [red] (11, -1) node[below right]{$v_n \sim \mu_n$} ;

\end{tikzpicture}}
\caption{\label{fig:advertisement} Online advertisement auction system.}
\end{figure*}

It is also closely related to online pricing or repeated posted price auctions. When a user wants to buy a flight ticket (or any other good), the selling company can learn the value distribution of the buyer and then dynamically adapts its prices in order to increase its revenue. The user can prevent this behavior in order to maximize her long term utility, even if it means refusing some apparently good offers in the short term (in poker lingo, she would be ``bluffing'').

As explained in Section~\ref{sec:toy} below, finding the best possible long term strategy is intractable, as the auctioneer could always adapt to the bidding strategy, leading to an arm race where each bidder and auctioneer successively adapts to the other one’s strategy. Such an arm race is instead avoided by trading-off between the best possible response to the auctioneer’s fixed strategy as well as the leaked quantity of information. The privacy loss then aims at bounding the incurred loss in bidder's utility if the auctioneer adapts her strategy using the revealed information.
\vspace{-0.5em}
\subsection{Learning through external servers}
\vspace{-0.5em}
Nowadays, several servers or clusters  allow their clients to perform heavy computations remotely,  for instance to learn some model parameters (say a deep neural net) for a given training set. The privacy concern when querying a server can sometimes be handled using homomorphic encryption \citep{gilad2016, bourse2018, sanyal2018}, if the cluster is designed in that way (typically a public model has been learned on the server). In this case, the client sends an encrypted testing set to the server, receives encrypted predictions and locally recovers the accurate ones. This technique, when available,  is  powerful, but requires heavy local computations. 

Consider instead a client wanting to learn a new model (say, a linear/logistic regression or any neural net)  on a dataset that has some confidential component.  Directly sending the  training set would reveal the whole data to the server owner, besides the risk of someone else observing it. The agent might instead prefer to send perturbed datasets, so that the computed model remains close to the accurate one, while keeping secret the true data. If the data contain sensitive information on individuals, then differential privacy is an appropriate solution. However, it is often the case that the private part is just a single piece of information of the client itself (say, its margin, its current wealth or its total number of users for instance) that is crucial to the final learned model but should not be totally revealed to a competitor.  Then differential privacy is no longer the solution, as there is only a single element to protect and/or to use.  Indeed, some privacy leakage is allowed and can lead to much more accurate parameters returned by the server and a higher utility at the end; the \NameOne aims at computing the best dataset to send to the server, in order to maximize the utility-privacy trade-off.
\vspace{-0.5em}
\section{Model}
\label{SE:model}
\vspace{-0.5em}
We first introduce a simple toy example in Section~\ref{sec:toy} giving insights into the more general problem, whose formal and general formulation is given in Section~\ref{sec:model}. 
\vspace{-0.5em}
\subsection{Toy Example}
\label{sec:toy}
 \vspace{-0.5em}
Suppose an agent is publicly playing an action $x \in \mathcal{X}$ to minimize a loss  $x^\top \!c_k$, where $c_k$ is some loss vector. The true type $k \in [K]$ is only known to the agent and drawn from a prior $p_0$. Without privacy concern, the agent would then solve for every $k$:
%\begin{equation}
%\label{eq:noprivacy}
\ $\min_{x \in \mathcal{X}} x^\top \!c_k.$ \\
%\end{equation}
Let us denote by $x^*_k$ the optimal solution of that problem. Besides maximizing her reward, the agent actually wants to protect the secret type $k$. After observing the action $x$ taken by the agent, an adversary can update her posterior distribution of the hidden type $p_x$. 

If the agent were to play deterministically $x^*_k$ when her type is $k$, then the adversary could infer the true value of $k$ based on the played action. The agent should instead choose her action randomly to hide her true type to the adversary. Given a type $k$, the strategy of the agent is then a probability distribution $\mu_k$ over $\mathcal{X}$  and her expected reward is $\mathbb{E}_{x \sim \mu_k} \big[ x^\top c_k\big]$. In this case, the posterior distribution after playing the action $x$ is computed using Bayes rule and if the different $\mu_k$ have overlapping supports, then the posterior distribution is no longer a Dirac mass, \ie some asymmetry of information is maintained.

The agent aims at simultaneously minimizing both the utility loss and the amount of information given to the adversary. A common way to measure the latter is given by the Kullback-Leibler (KL) divergence between the prior and the posterior \citep{sims2003}: $\mathrm{KL}(p_x, p_0) = \sum_{k=1}^K \log\left( \frac{p_x(k)}{p_0(k)} \right) p_x(k)$, where $p_x(k)= \frac{p_0(k) \mu_k(x)}{\sum_{l=1}^K p_0(l) \mu_l(x)}$. If the information cost scales in utility with $\lambda>0$, the regularized  loss of the agent is then  $x^\top \!c_k + \lambda \mathrm{KL}(p_x, p_0)$ instead of $x^\top \!c_k$. Overall, the global objective of the agent is the following minimization:
%of $\sum_{k=1}^K p_0(k) \mathbb{E}_{x \sim \mu_k} \big[ x^\top \!c_k + \lambda \mathrm{KL}(p_x, p_0)\big]$ over $\mu_1,\ldots,\mu_K$.
\begin{small}\begin{equation*}
\min\limits_{\mu_1, \ldots, \mu_K} \sum\limits_{k=1}^K p_0(k) \mathbb{E}_{x \sim \mu_k} \big[ x^\top \!c_k + \lambda \mathrm{KL}(p_x, p_0)\big].
\end{equation*}\end{small}%
In the limit case $\lambda = 0$, the agent follows a totally revealing strategy and deterministically plays $x^*_k$ given~$k$. When $\lambda = \infty$, the agent focuses on perfect privacy and looks for the best  action  chosen independently of the type: $x \indep k$. It  corresponds to a so called non-revealing strategy in game theory and the best strategy is then to play $\argmin_{x} x^\top \! c[p_0]$ where $c[p_0] = \sum_{k=1}^K p_0(k) c_k$. For a positive $\lambda$, the behavior of the player will then interpolate between these two extreme strategies.
%, where different players have private information and try to maximize their utility, which also depends on some private information \cite{aumann1995}. 

This problem is related to repeated games with incomplete information \citep{aumann1995}, where players have private information affecting their utility functions. Playing  some action leaks information to the other players, who then change their strategies in consequence. The goal is then to control the amount of information leaked to the adversaries in order to maximize one's own utility.
%In this case, a usual strategy consists in a splitting strategy, \ie a player will generate some posterior probability of her type and will maximize her reward in expectation over this posterior. 
In practice, it can be impossible to compute the best adversarial strategy, \eg the player is unaware of how the adversaries would~adapt. The utility loss caused by adversarial actions is then modeled as a function of the amount of revealed information.
\vspace{-0.5em}
\subsection{General model}
\vspace{-0.5em}
\label{sec:model}
We now introduce formally the general model sketched by the previous toy example. The agent (or player) has a private type $y \in \mathcal{Y}$ drawn according to a prior $p_0$ whose support can be infinite. She then chooses an action $x \in \mathcal{X}$ to maximize her utility, which depends on both her action and her type. Meanwhile, she wants to hide the true value of her type $y$. A strategy is thus a mapping $\mathcal{Y} \to \mathcal{P}(\mathcal{X})$, where $\mathcal{P}(\mathcal{X})$ denotes the set of distributions over $\mathcal{X}$; for the sake of conciseness, we denote by ${X|Y \in \mathcal{P}(\mathcal{X})^{\mathcal{Y}}}$ such a strategy. In the toy example, this mapping was given by $k \mapsto \mu_k$.
The adversary observes her action $x$ and tries to infer the type of the agent. We assume a perfect adversary, \ie she can exactly compute the posterior distribution $p_x$.

Let $c(x,y)$ be the utility loss for playing $x \in \mathcal{X}$ with the type $y \in \mathcal{Y}$. 
The cost of information is $c_{\text{priv}}(X, Y)$ where $(X, Y)$ is the joint distribution of the action and the type. %The cost functions $c$ and $c_{\text{priv}}$ are both assumed to be known. 
In the toy example given in Section~\ref{sec:toy}, the utility cost was given by $c(x, k) = x^\top c_k$ and the privacy cost was the expected KL divergence between $p_x$ and $p_0$. 
%The privacy cost of differential privacy is 
%%\begin{equation*}c_{\text{priv}}(X, Y) = \max\limits_{\substack{(x, y, y')\\ d_h(y,y')=1}} \Big|\log\Big(\frac{\mathbb{P}(x\ |\ Y=y))}{\mathbb{P}(x\ |\ Y=y'))}\Big)\Big|, \end{equation*}
%\begin{small} \begin{equation*} c_{\text{priv}}(X, Y) = \max\limits_{(x, y, y')} \Big\{ \big|\log \frac{\mathbb{P}(x\ |\ Y=y))}{\mathbb{P}(x\ |\ Y=y'))}\big|\ ;\ d_h(y,y')=1  \Big\} \end{equation*} \end{small}where $d_h$ is the Hamming distance \cite{dwork2011}.
The previous frameworks aimed at minimizing the utility loss with a privacy cost below some threshold $\varepsilon > 0$, \ie minimize $\mathbb{E}_{(x,y) \sim (X,Y)}\big[ c(x,y) \big]$ such that $c_{\text{priv}}(X, Y) \leq \varepsilon$.
%\citep{Mu2019}. 
Here, this privacy loss has some utility scaling with $\lambda>0$, which can be seen as the value of information.
The final objective of the agent is then to minimize the following loss:
\begin{small}\begin{equation}
\label{eq:staticmin0}
\inf\limits_{X|Y \in \mathcal{P}(\mathcal{X})^{\mathcal{Y}}} \mathbb{E}_{(x,y) \sim (X,Y)}\big[ c(x,y) \big] +  \lambda \ c_{\text{priv}}(X, Y).
\end{equation}\end{small}%
As mentioned above, the cost of information is here defined as a measure between the posterior $p_x$ and the prior distribution $p_0$ of the type, \ie $c_{\text{priv}}(X, Y) = \mathbb{E}_{x \sim X} D(p_x, p_0)$ for some function $D$\footnote{We here favor ex-ante costs as they suggest that the value of information can be heterogeneous among types.}. In the toy example of Section~\ref{sec:toy}, $D(p_x, p_0) = \mathrm{KL}(p_x, p_0)$, which is a classical cost of information in economics.

For a distribution $\gamma \in \mathcal{P}(\mathcal{X}\times\mathcal{Y})$, we denote by $\pi_{1 \#} \gamma$ (resp. $\pi_{2 \#} \gamma$) the marginal distribution of $X$ (resp. $Y$):
$
\pi_{1 \#} \gamma(A) = \gamma(A \times \mathcal{Y}) \text{ and } \pi_{2 \#} \gamma(B) = \gamma(\mathcal{X} \times B).$ In order to have a simpler formulation of the problem, we remark that instead of defining a strategy by the conditional distribution $X|Y$, it is equivalent to see it as a joint distribution $\gamma$ of $(X,Y)$ with a marginal over the type equal to the prior: $\pi_{2 \#} \gamma = p_0$. 
The remaining of the paper focuses on the problem below, which we call \textbf{\NameOne[.]} With the privacy cost defined as above, the minimization problem~\eqref{eq:staticmin0} is equivalent to 
\begin{small}
\begin{equation}
\label{eq:staticmin1} \tag{PRP}
\inf_{\substack{\gamma \in \mathcal{P}\left(\mathcal{X} \times \mathcal{Y} \right) \\ \pi_{2 \#} \gamma = p_0}} \int_{\mathcal{X} \times \mathcal{Y}} \!\!\!\!\!\! \left[ c(x,y) + \lambda \ D \! \left( p_x, p_0 \right) \right] \mathrm{d}\gamma(x,y).
\end{equation}\end{small}
\vspace{-1.5em}
\section{A convex minimization problem}
\label{sec:convex}
\vspace{-0.5em}
In this section, we study some theoretical properties of the Problem~\eqref{eq:staticmin1}. We first recall the definition of an $f$-divergence.

\begin{definition}
$D$ is an $f$-divergence if for all distributions $P, Q$ such that $P$ is absolutely continuous w.r.t.~$Q$, $D (P, Q) = \int_\mathcal{Y}f\left( \frac{\mathrm{d}P(y)}{\mathrm{d}Q(y)}\right)\mathrm{d}Q(y)$ where $f$ is a convex function defined on $\mathbb{R}_+^*$ with $f(1)=0$.
\end{definition}
%The assumption of absolute continuity is not significant here, as the posterior is always absolutely continuous with respect to the prior. 
The set of $f$-divergences includes common divergences such as the Kullback-Leibler divergence, the reverse Kullback-Leibler or the Total Variation distance. 

Also, the min-entropy defined by $D(P, Q) = \log\left(\esssup \mathrm{d}P/\mathrm{d}Q\right)$ is widely used for privacy \citep{toth2004, Smith2009}. It corresponds to the limit of the Renyi divergence  $\ln\left( \sum_{i=1}^n p_i^\alpha q_i^{1-\alpha} \right)/(\alpha-1)$, when $\alpha \to + \infty$ \citep{renyi1961, mironov2017}. Although it is not an $f$-divergence, the Renyi divergence derives from the $f$-divergence associated to the convex function $t \mapsto (t^{\alpha}-1)/(\alpha-1)$.
%The set of $f$-divergences does not include all the possible privacy costs, but it still is a large part of them, including the classical cost of information given by the KL. 
$f$-divergence costs have been recently considered in the computer science literature in a non-Bayesian case and then present the good properties of convexity, composition and post-processing invariance \citep{chaudhuri2019}.

In the remaining of the paper, $D$ is an $f$-divergence. \eqref{eq:staticmin1} then becomes a convex minimization problem.
\begin{thm}
\label{thm:convexprob}
If $D$ is an $f$-divergence, \eqref{eq:staticmin1} is a convex problem in $\gamma \in \mathcal{P}(\mathcal{X} \! \times \! \mathcal{Y})$\footnote{It is convex in a usual sense and not geodesically here.}.
\end{thm}
The proof is given in Appendix~\ref{app:proofs}. Although $\mathcal{P}(\mathcal{X} \times \mathcal{Y})$ has generally an infinite dimension, it is dimensionally finite if both sets $\mathcal{X}$ and $\mathcal{Y}$ are discrete. A minimum can then be found using classical optimization methods such as gradient descent. In the case of low dimensional spaces $\mathcal{X}$ and $\mathcal{Y}$, they can be approximated by finite grids. However, the size of the grid grows exponentially with the dimension and another approach is needed for large dimensions of $\mathcal{X}$ and $\mathcal{Y}$.
\vspace{-0.5em}
\subsection{Discrete type space}
\label{sec:finite}
We assume here that $\mathcal{X}$ is an infinite action space and $\mathcal{Y}$ is of cardinality $K$ (or equivalently, that $p_0$ is  a discrete prior of size $K$), so that $p_0 = \sum_{k=1}^K p_0^k \delta_{y_k}$. %Notice that most of the continuous distributions are in practice approximated by their empirical observations, which define a discrete distribution.
For a fixed joint distribution $\gamma$, let the measure $\mu_k$ be defined for any $A \subset \mathcal{X}$ by $\mu_k(A) = \gamma(A \times \lbrace y_k \rbrace)$ and $\mu = \sum_{k=1}^K \mu_k =\pi_{1 \#} \gamma$. The function $p^k(x) = \frac{\mathrm{d} \mu_k(x)}{\mathrm{d}\mu(x)}$, defined over the support of $\mu$ by absolute continuity, is the posterior probability of having the type $k$ when playing $x$. In this specific setting, the tuple $(\mu, (p^k)_k)$ exactly determines $\gamma$. \eqref{eq:staticmin1} is then equivalent to:
\begin{small}
\begin{equation*}
%\label{eq:staticfinite1}
\begin{aligned}
\inf\limits_{\substack{\mu, (p^k(\cdot))_{1\leq k \leq K}\\ p^k \geq 0, \sum_{l=1}^K p^l(\cdot) = 1}} \!\!\!\!\!\! &
 \sum\limits_k \int_\mathcal{X}\big[ p^k(x) c(x, y_k) + \lambda p_0^k f\left( \frac{p^k(x)}{p_0^k}\right) \big] \mathrm{d}\mu(x) \\ \text{such that } & \forall k \leq K, \int_\mathcal{X} p^k(x) \mathrm{d}\mu(x) = p_0^k .
\end{aligned}
\end{equation*}\end{small}For fixed posterior distributions $p^k$, this is a generalized moment problem on the distribution $\mu$ \citep{lasserre2001}. The same types of arguments can then be used for the existence and the form of optimal solutions. 
%The proofs of Theorems~\ref{thm:caratheodory1} and \ref{thm:caratheodory2} are delayed to Appendix~\ref{app:proofs}.

\begin{thm}
\label{thm:caratheodory1}
If the prior is dicrete of size $K$, for all $\varepsilon > 0$, \eqref{eq:staticmin1} has an $\varepsilon$-optimal solution such that $\pi_{1 \#} \gamma = \mu$ has a finite support of at most $K+2$ points.  \\
Furthermore, if $\mathcal{X}$ is compact and $c(\cdot , y_k)$ is lower semi-continuous for every $k$, then it also holds for $\varepsilon=0$.
\end{thm}

The proof is delayed to Appendix~\ref{app:proofs}. If the support of  $\gamma$ is included in $\lbrace (x_i, y_k) \ | \ 1 \leq i \leq K+2,\ 1 \leq k \leq K \rbrace$, we will denote it as a matrix $\gamma_{i,k} \coloneqq \gamma(\lbrace (x_i,y_k) \rbrace)$.
\begin{coro}
\label{coro:eqproblem}
In the case of a discrete prior, \eqref{eq:staticmin1} is equivalent to:
\begin{small}
\begin{gather*}
\inf\limits_{(\gamma, x) \in \mathbb{R}_+^{(K+2) \! \times \! K} \times \mathcal{X}^{K+2}}  \sum_{i,k} \gamma_{i,k} \ c(x_i, y_k) + \lambda \sum_{i,k} \gamma_{i,k} D(p_{x_i}, p_0)  \\
\text{such that }  \forall k \leq K, \ \sum_{i} \gamma_{i,k} = p_0^k. 
\end{gather*}\end{small}
\end{coro}
\vspace{-0.5em}
%Theorem~\ref{thm:caratheodory1} claims that \eqref{eq:staticmin1} is equivalent to this problem if we also impose no redundancy, \ie $x_i \neq x_j$ for $i \neq j$. The proof of Lemma~\ref{lemma:eqproblem}, which is delayed to the Appendix~\ref{app:proofs}, shows that for any redundant solution $(\gamma, x)$, there is a non-redundant version with a lower value, using the subadditivity in $\gamma$.
Although it seems easier to consider the dimensionally finite problem given by Corollary~\ref{coro:eqproblem}, it is not jointly convex in $(\gamma, x)$.
No general algorithms exist to efficiently minimize non-convex problems. We refer the reader to \citep{horst2000} for an introduction to non-convex optimization.

The remaining of the paper reformulates the problem to better understand its structure, which then leads to better local minima. Computing global minima of Problem~\eqref{eq:staticmin1} is yet left open for future work.

\vspace{-0.5em}
\section{Sinkhorn Loss minimization}
\label{sec:transport}
\vspace{-0.5em}
Formally, \eqref{eq:staticmin1} is expressed as Optimal Transport Minimization for the utility cost $c$ with a regularization given by the privacy cost. In this section, we focus on the case where this privacy cost is the Kullback-Leibler divergence. In this case, the problem becomes a Sinkhorn loss minimization, which presents computationally tractable schemes \citep{OTbook}. If the privacy cost is  the KL divergence between the posterior and the prior, \ie $f(t) = t \log(t)$, then the regularization term corresponds to the mutual information $I(X;Y)$. As explained above, this is the classical cost of information in economics.
%The mutual information is the natural choice for a $f$-divergence privacy cost. 

Recall that the Sinkhorn loss for given distributions $(\mu, \nu) \in \mathcal{P}(\mathcal{X}) \times \mathcal{P}(\mathcal{Y})$ is defined by 
\begin{small}\begin{equation} \label{eq:sinkhorn} 
\begin{aligned}
\mathrm{OT}_{c, \lambda}(\mu, \nu) \coloneqq & \min\limits_{\gamma \in \Pi(\mu, \nu)} \int c(x,y) \mathrm{d}\gamma(x,y)  \\& + \lambda \int \log\left( \frac{\mathrm{d}\gamma(x,y)}{\mathrm{d}\mu(x) \mathrm{d}\nu(y)}\right) \mathrm{d}\gamma(x,y), 
\end{aligned} \end{equation}\end{small}%\begin{equation} \label{eq:sinkhorn} \begin{split}
%\mathrm{OT}_{c, \lambda}(\mu, \nu) := \min\limits_{\gamma \in \Pi(\mu, \nu)} \int c(x,y) \mathrm{d}\gamma(x,y) + \lambda \int \log\left( \frac{\mathrm{d}\gamma(x,y)}{\mathrm{d}\mu(x) \mathrm{d}\nu(y)}\right) \mathrm{d}\gamma(x,y), \\
%\text{with } \Pi(\mu, \nu) = \lbrace \gamma \in \Delta(\mathcal{X}\times\mathcal{Y}) \ | \ \pi_1 \# \gamma = \mu \text{ and } \pi_2\#\gamma = \nu \rbrace.
%\end{split} \end{equation}
where $\Pi(\mu, \nu) = \lbrace \gamma \in \mathcal{P}(\mathcal{X}\times\mathcal{Y}) \ | \ \pi_{1\#}\gamma=\mu  \text{ and } \pi_{2\#}\gamma = \nu \rbrace$. Problem~\eqref{eq:staticmin1} with $D=\mathrm{KL}$ can then be rewritten as the following Optimal Transport minimization:
\begin{small}\vspace{-0.5em}\begin{equation*} %\label{eq:OTproblem}
\inf\limits_{\mu\in\mathcal{P}(\mathcal{X})} \mathrm{OT}_{c, \lambda}(\mu, p_0).\vspace{-0.5em}
\end{equation*}\end{small}%
%\vspace{-0.5em}%
Indeed, observe that $\frac{\mathrm{d}\gamma(x,y)}{\mathrm{d}\mu(x)}$ is the posterior probability $\mathrm{d}p_x(y)$, thanks to Bayes rule. The regularization term in equation~\eqref{eq:sinkhorn} then corresponds to $D(p_x, p_0)$ as $p_0 = \nu$ and $D=\mathrm{KL}$ here. The minimization problem given by equation~\eqref{eq:sinkhorn} is thus equivalent to equation~\eqref{eq:staticmin1} with the additional constraint $\pi_{1 \#} \gamma = \mu$. Minimizing without this constraint is thus equivalent to minimizing the Sinkhorn loss over all action distributions $\mu$.

\medskip

While the regularization term is usually only added to speed up the computations, it here directly appears in the cost of the original problem since it corresponds to the privacy cost! An approximation of $\mathrm{OT}_{c, \lambda}(\mu, \nu)$ can then be quickly computed for discrete distributions using Sinkhorn algorithm \citep{Cuturi2013}.
%However, if the scaling parameter $\lambda$ is of the order $\Theta(1)$  (instead of being close to $0$) then the Sinkhorn algorithm convergence rate and sample complexity are both comparable with those of Maximal Mean Discrepancy \cite{Gretton08, Cuturi2013, genevay2018b}.\todo{Check if correct}

Notice that the definition of Sinkhorn loss sometimes differs in the literature and instead considers $\int \log\left( \mathrm{d}\gamma(x,y)\right) \mathrm{d}\gamma(x,y)$ for the regularization term. When $\mu$ and $\nu$ are both fixed, the optimal transport plan $\gamma$ remains the same. As $\mu$ is varying here, these notions yet become different. For this alternative definition, a minimizing distribution $\mu$ would actually be easy to compute. It is much more complex in our problem because of the presence of $\mu$ in the denominator of the logarithmic term.

In the case of discrete support, we can then look for a distribution $\mu= \sum_{j=1}^{K+2} \alpha_j \delta_{x_j}$. In case of continuous distributions, they could still be approximated using sampled discrete distributions as previously done for generative models \citep{genevay2018b, genevay18a}.

Besides being a new interpretation of Sinkhorn loss, this reformulation mainly allows to better understand the problem structure and reduce the support size of the distribution in the minimization problem.

%
%If the $f$-divergence is not the Kullback-Leibler, then , an approximation of $\mathrm{OT}_{c, \lambda}(\mu, \nu)$ can be computed using gradient descent, thanks to the convexity of the problem. It would however not present the same computational advantages as with the entropic regularization. 
%
%In that case, directly using gradient descent on Problem~\ref{eq:staticfinite2} might even be the best option, but the parameters ($(\gamma, x)$ instead of $(\mu, x)$) would lie in a larger space and could lead to spurious local minima, as empirically confirmed in Section~\ref{sec:expe}. \todo{Pas compris ce paragraphe. On l'enleve ?}
\vspace{-0.5em}
\subsection{Minimization algorithm}
\vspace{-0.5em}
We now consider the following minimization problem over the tuple $(\alpha, x)$:
\begin{small}\vspace{-1em}
\begin{equation}
\label{eq:sinkhmini}
\inf\limits_{(\alpha, x) \in \Delta_{K+2} \! \times \! \mathcal{X}^{K+2}} \mathrm{OT}_{c, \lambda}\Big( \sum_{i=1}^{K+2} \alpha_i \delta_{x_i}, p_0 \Big).\vspace{-0.5em}
\end{equation}\end{small}%
The main difficulties come from the computation of the objective function and its gradient to use classical gradient based methods.
\vspace{-0.5em}
\paragraph{Sinkhorn algorithm.}
It was recently suggested to use the Sinkhorn algorithm, which has a linear convergence rate, to compute $\mathrm{OT}_{c, \lambda} (\mu, \nu)$ for distributions $\mu = \sum_{i=1}^n \alpha_i \delta_{x_i}$ and $\nu = \sum_{j=1}^m \beta_j \delta_{y_j}$ \citep{Knight2008, Cuturi2013}.
Let $K$ be the exponential cost matrix defined by $K_{i,j} = e^{-\frac{c(x_i, y_j)}{\lambda}}$. In the discrete case, the unique matrix $\gamma$ solution of the Problem~\eqref{eq:sinkhorn} has the form $\mathrm{diag}(u) K \mathrm{diag}(v)$. The Sinkhorn algorithm then updates $(u,v) \gets (\alpha / Kv , \beta / K^\top u)$ (with component-wise division) for $n$ iterations or until convergence.
\vspace{-0.5em}
\paragraph{Gradient computation.}
Computing $\nabla\mathrm{OT}_{c, \lambda}$ is a known difficult task \citep{Feydy2018, Luise2018, genevay18a}. A simple solution consists in using automatic differentiation, \ie computing the gradient using the chain rule over the simple successive operations computed during the Sinkhorn algorithm.

The gradient can also be computed from the dual solution of Problem~\eqref{eq:sinkhorn}. This method is faster as it does not need to store all the Sinkhorn iterations in memory and \textit{backpropagate} through them afterwards. Convergence of Sinkhorn algorithm has yet to be guaranteed to provide an accurate approximation of the gradient \citep[see][for an extended discussion]{OTbook}. Automatic differentiation is used in the experiments because of this last reason.

%\begin{algorithm}[H]
%\caption{\label{algo:optim} Optimization scheme}
%\textbf{Input:} discrete prior $p_0$, cost function $c$ \\
%\textbf{Output:} strategy $\mu_\theta$
%\begin{algorithmic}
%\STATE Initialize $\theta$
%\STATE \textbf{Until convergence }\algorithmicdo
%\STATE \hspace{\algorithmicindent} $\mathrm{cost} \gets \mathrm{Sinkhorn}_{n}(\mu_{\theta}, p_0, c, \lambda)$ \\
%\hfill \COMMENT{$n$ iterations of Sinkhorn algorithm}
%\STATE \hspace{\algorithmicindent} $\nabla_{\theta} \gets \mathrm{AutoDiff}_{\theta}(\mathrm{loss})$
%\STATE \hspace{\algorithmicindent} $\theta \gets \mathrm{Update}(\theta, \nabla_{\theta})$ \hfill \COMMENT{gradient descent}
%\STATE \algorithmicend
%\RETURN $\mu_\theta$
%\end{algorithmic}
%\end{algorithm}

\vspace{-1em}
\section{Experiments and particular cases}
\label{sec:expe}
\vspace{-1em}
In this section, the case of linear utility cost is first considered and shown to have relations with DC programming, which allows efficient algorithms. The performances of different optimization schemes are then compared on a simple example. Simulations based on the Sinkhorn scheme are then run for the real problem of online repeated auctions. The code is available at \url{github.com/eboursier/regularized_private_learning}.
\vspace{-0.5em}
\subsection{Linear utility cost}
\vspace{-0.5em}
Section~\ref{sec:convex} described a general optimization scheme for Problem~\eqref{eq:staticmin1} with a discrete type prior. It used a dimensionally finite, non-convex problem. An objective is then to find a local minimum.
Local minima can be found using classical techniques of gradient descent \citep{Wright2015}. However in some particular cases, better schemes are possible as claimed in Section~\ref{sec:transport} for the particular case of entropic regularization. In the case of a linear utility for any privacy cost, it is related to DC programming \citep{horst2000}. A standard DC program is of the form
$ %\begin{equation*}
\min_{x \in \mathcal{X}} f(x) - g(x),
$ %\end{equation*}
where both $f$ and $g$ are convex functions. Specific optimization schemes are then possible \citep{tao1997, horst1999, horst2000}. In the case of linear utility costs over a hyperrectangle, \eqref{eq:staticmin1} can be reformulated as a DC program stated in Theorem~\ref{thm:linearcost}. Its proof is delayed to Appendix~\ref{app:proofs}.

\begin{thm}
\label{thm:linearcost}
If $\mathcal{X} = \prod\limits_{l=1}^d [a_l,\ b_l]$ and $c(x,y) = x^\top y$, then \eqref{eq:staticmin1} is equivalent to the following DC program:
\begin{small}
\begin{gather*}
\min_{\gamma \in \mathbb{R}_+^{(K+2) \times K}}  \lambda \sum\limits_{i,k} p_0^k h_k(\gamma_i) - \sum_{i=1}^{K+2} \left\| \sum_{k=1}^K \gamma_{i,k} \phi(y_k) \right\|_1 \! , \\ \text{such that }  \forall k \leq K, \ \sum_{i=1}^{K+2} \gamma_{i,k} = p_0^k ,
\end{gather*}\end{small}%
with $\phi(y)^l \coloneqq (b_l - a_l)y^l/2$ \\ and $h_{k}(\gamma_i) \coloneqq \big(\sum_{m=1}^K \gamma_{i,m}\big) f \big(\frac{\gamma_{i,k}}{p_0^k\sum_{m=1}^K \gamma_{i,m}}\big)$.
%\begin{equation*}
%\min\limits_{\substack{\gamma \in \mathbb{R}_+^{(K+2) \times K} \\ \forall k, \sum\limits_{i=1}^{K+2} \gamma_{i,k} = p_0^k}} \lambda \sum\limits_{i,k} p_0^k h_k(\gamma_i) - \sum\limits_{i} \| \sum\limits_k \gamma_{i,k} \phi(y_k) \|_1, \hspace{1cm} \text{ where } \phi(y)^l := (b_l - a_l) \frac{y^l}{2}.
%\end{equation*}
\end{thm}
More generally, if the cost $c$ is concave and the action space $\mathcal{X}$ is a polytope,  optimal actions are  located on the vertices of $\mathcal{X}$. In that case, we can therefore replace $\mathcal{X}$ by the set of its vertices and the problem becomes a dimensionally finite convex problem as already claimed in Section~\ref{sec:model}. Unfortunately, for some polytopes such as hyperrectangles, the number of vertices grows exponentially with the dimension and the optimization scheme is no longer tractable in large dimensions.
\vspace{-0.5em}
\subsection{Comparing methods on the toy example}
\vspace{-0.5em}
\begin{figure*}[t!]
\begin{adjustwidth}{-0.7in}{-0.2in}
\centering
\resizebox{0.9\linewidth}{!}{\input{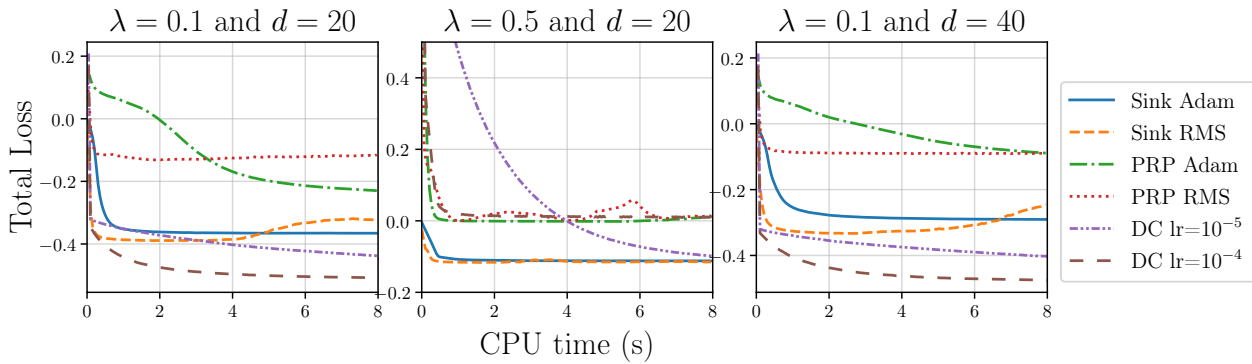}}
\caption{\label{fig:simu1}Comparison of optimization schemes.}
\end{adjustwidth}
\vspace{-0.5em}
\end{figure*}

%We first compare the performance of different algorithms on the toy example. 
We consider the linear utility loss $c(x,y) = x^\top y$ over the space $\mathcal{X} = [-1, 1]^d$ and the Kullback-Leibler divergence for privacy cost, so that both DC and Sinkhorn schemes are possible.
Different methods exist for DC programming and they compute either a local or a global minimum. We here choose the DCA algorithm \citep{tao1997} as it computes a local minimum and is thus comparable to the other considered schemes. Figure~\ref{fig:simu1} compares, for different problem parameters, the convergence rates of usual non-convex optimization methods (ADAM and RMS), as well as DCA. The former methods are used on different minimizations given by Corrolary~\ref{coro:eqproblem} and Equation~\eqref{eq:sinkhmini} (resp.~PRP and Sink).
%
%\begin{figure}[h]
%\begin{adjustwidth}{-0.7in}{-0.2in}
%\begin{subfigure}{.35\linewidth}
%\centering
%\resizebox{\linewidth}{!}{\input{lamb01_K100_dim20.pgf}}
%\caption{$\lambda=0.1$, $K=100$ and $d=20$.}
%\label{fig:sub11}
%\end{subfigure}%
%\begin{subfigure}{.35\linewidth}
%\centering
%\resizebox{\linewidth}{!}{\input{lamb05_K100_dim20.pgf}}
%\caption{$\lambda=0.5$, $K=100$ and $d=20$.}
%\label{fig:sub12}
%\end{subfigure}
%\begin{subfigure}{.35\linewidth}
%\centering
%\resizebox{\linewidth}{!}{\input{lamb01_K100_dim40.pgf}}
%\caption{$\lambda=0.1$, $K=100$ and $d=40$.}
%\label{fig:sub13}
%\end{subfigure}
%\caption{\label{fig:simu1}Comparison of different optimization schemes on toy examples with tuned learning rates (\texttt{lr}).}
%\end{adjustwidth}
%\end{figure}

We optimized using projected gradient descent for well tuned learning rates. The prior $p_0^k$ is chosen proportional to $e^{Z_k}$ for any $k \in [K]$, where $Z_k$ is drawn uniformly at random in $[0, 1]$. Each $y_i^k$ is taken uniformly at random in $[-1, 1]$ and is rescaled so that $\| y_i \|_1 = 1$. The values are averaged over $200$ runs.
% compares the computation times for different values of the regularization constant $\lambda$, the action space dimension $d$ with the size of the prior $K=100$. 

The DC method finds better local minima than the other ones. This  was already observed in practice \citep{tao1997} and confirms that it is more adapted to the structure of the problem, despite being only applicable in very specific cases such as linear cost on hyperrectangles. Also, the PRP method converges to worse spurious local minima as it optimizes in higher dimensional spaces than the Sinkhorn method. We also observed in our experiments that PRP method is more sensitive to problem parameters than Sinkhorn method.

The Sinkhorn method seems to perform better for larger values of $\lambda$. Indeed, given the actions, the Sinkhorn method computes the best joint distribution for each iteration and thus performs well when the privacy cost is predominant, while DCA computes the best actions given a joint distribution and thus performs well when the utility cost is predominant. %Finally, the Sinkhorn method seems to perform better than the descent method, while the DC method performs even better as it takes even more into account the problem structure for this particular case. 
It is thus crucial to choose the method which is most adapted to the problem structure as it can lead to significant improvement in the solution.
%
%\begin{figure}[h]
%\centering
%\resizebox{0.8\linewidth}{!}{\input{tradeoff_K10_v2.pgf}}
%\caption{Evolution of privacy-utility with $\lambda$.}
%\label{fig:sub21}
%\end{figure}
%
\begin{figure*}[t!]
\begin{subfigure}{.45\linewidth}
\centering
\resizebox{\linewidth}{!}{\input{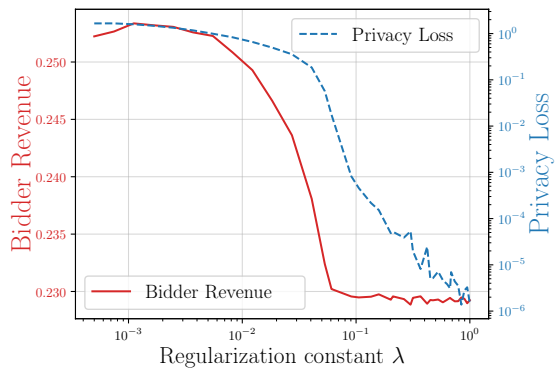}}
\caption{Evolution of privacy-utility with $\lambda$.}
\label{fig:sub21}
\end{subfigure}%
\hspace{1cm}
\begin{subfigure}{.45\linewidth}
\centering
\resizebox{\linewidth}{!}{\input{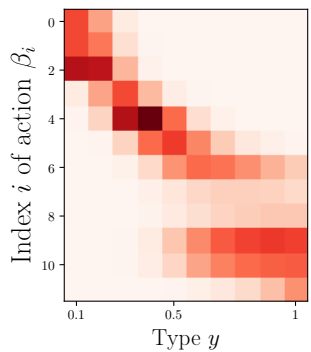}}
\caption{Joint distribution map for $\lambda = 0.01$. The intensity of a point $(i,j)$ corresponds to the value of $\gamma(\beta_i, y_j)$.}
\label{fig:sub22}
\end{subfigure}
\caption{\label{fig:simu2}Privacy-utility trade-off in online repeated auctions.}
\end{figure*}
%
%\begin{figure}[h]
%\centering
%\resizebox{0.8\linewidth}{!}{\input{gamma_heatmap_lamb0001.pgf}}
%\caption{Joint distribution heat-map, with $\lambda = 0.01$.}
%\label{fig:sub22}
%\end{figure}
%
\begin{figure*}[t!]
\resizebox{0.9\linewidth}{!}{\input{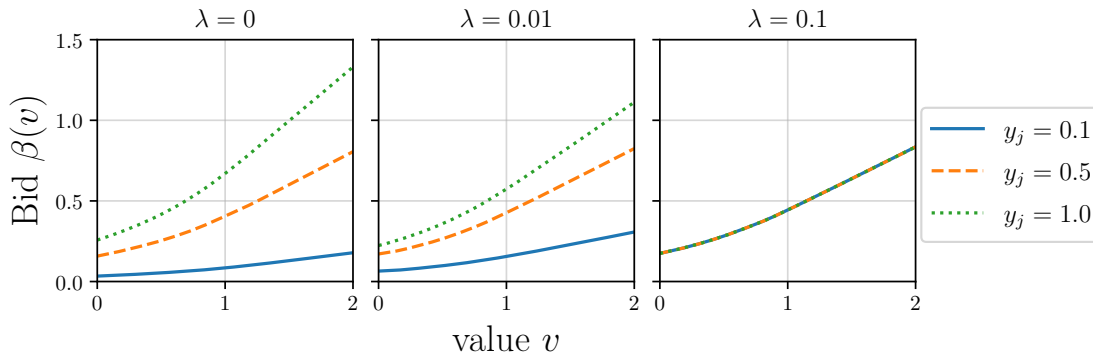}}
\caption{\label{fig:simu3}Evolution of the bidding strategy with the type and the regularization constant.}
\end{figure*}
\vspace{-0.5em}
\subsection{Utility-privacy in repeated auctions}
\vspace{-0.5em}
For repeated second price auctions following a precise scheme \citep{paes2016}, there exist numerical methods to implement an optimal strategy for the bidder \citep{nedelec2019}. However, if the auctioneer knows that the bidder plays this strategy, he can still infer the bidder's type and adapt to it. We thus require to add a privacy cost to avoid this kind of behavior from the auctioneer.

For simplicity, bidder's valuations are assumed to be exponential distributions, so that the private type $y$ corresponds to the only parameter of this distribution, i.e., its expectation: $y = \mathbb{E}_{v \sim \mu_y}[v]$. Moreover, we assume that the prior $p_0$ over $y$ is the discretized uniform distribution on $[0, 1]$ with a support of size $K=10$; let $\{y_j\}_{j=1,\ldots,K}$ be the support of $p_0$.

In repeated auctions, values  $v$ are repeatedly sampled from the distribution $\mu_{y_j}$ and a bidder policy  is a mapping $\beta(\cdot)$ from values to bids, i.e., she bids $\beta(v)$ if her value is $v$. So a type $y_j$ and a policy $\beta(\cdot)$ generate the bid distribution $\beta_{\#} \mu_{y_j}$, which corresponds to an action in $\mathcal{X}$ in our setting. As a consequence, the set of actions of the agent are the probability distributions over $\mathbb{R}_+$ and an action $\rho_i$  is naturally generated from the valuation distribution via the optimal monotone transport map denoted by $\beta_j^i$,  i.e., $\rho_i =  \beta^i_{j \#} \mu_{y_j}$ \citep{santambrogio2015}. In the particular case of exponential distributions, this implies that $\beta_i^j(v) = \beta_i(v/y_j)$ where $\beta_i$ is the unique monotone transport map from $\mathrm{Exp}(1)$ to $\rho_i$. 
The revenue of the bidder is then deduced for exponential distributions \citep{nedelec2019} as 
\begin{small}
\begin{align*}&r(\beta_i, y_j) = 1-c(\beta_i, y_j) \\& = \mathbb{E}_{v \sim \mathrm{Exp}(1)} \big[ \big( y_j v - \beta_i(v) + \beta'_i(v)\big) G\big(\beta_i(v)\big) \mathds{1}_{\beta_i(v) - \beta'_i(v)\geq 0}\big],
\end{align*} 
\end{small}%
where $G$ is the \textit{c.d.f.} of the maximum bid of the other bidders. We here consider a single truthful opponent with a uniform value distribution on $[0, 1]$, so that $G(x) = \min(x,1)$. This utility is averaged over $10^3$ values drawn from the corresponding distribution at each training step and $10^6$ values for the final evaluation.

Considering the KL for privacy cost, we compute a strategy $(\gamma, \beta)$ using the Sinkhorn scheme described in Section~\ref{sec:transport}. Every action $\beta_i$ is parametrized as a single layer neural network of $100$ ReLUs. Figure~\ref{fig:sub21} represents both utility and privacy as a function of the regularization factor $\lambda$.

Naturally, both the bidder revenue and the privacy loss decrease with $\lambda$, going from revealing strategies for $\lambda \simeq 10^{-3}$ to non-revealing strategies for larger $\lambda$. They significantly drop at a critical point near $0.05$, which can be seen as the cost of information here. There is a $8$\% revenue difference\footnote{Which is significant for large firms such as those presented in Figure~\ref{fig:advertisement} besides the revenue difference brought by considering non truthful strategies \citep{nedelec2019}.} between the non revealing strategy and the partially revealing strategy shown in Figure~\ref{fig:sub22}. The latter randomizes the type over the neighboring types and reveals more information when the revenue is sensible to the action, \ie for low types $y_j$ here. This strategy thus takes advantage from the fact that the value of information is here heterogeneous among types, as desired in the design of our model.

Figure~\ref{fig:simu3} shows the most used action for different types and $\lambda$. In the revealing strategy ($\lambda=0$), the action significantly scales with the type. But as $\lambda$ grows, this rescaling shrinks so that the actions perform for several types, until having a single action in the non-revealing strategy. This shrinkage is also more important for large values of $y_j$. This confirms the observation made above: the player loses less by hiding her type for large values than for low values and she is thus more willing to hide her type when it is large.

Besides confirming expected results, this illustrates how the \NameOne is adapted to complex utility costs and action spaces, such as distributions or function spaces.

\vspace{-0.5em}
\section{Conclusion}
\vspace{-0.5em}
%In case of a continuous prior, Caratheodory's theorem is no more valid and the problem becomes a continuous Optimal Transport problem. This problem is way harder as computing the Sinkhorn divergence only is a difficult task \cite{genevay2016}. However, as soon as this distance and a gradient can be approximated, a parametrized locally optimal $\mu_{\theta}$ can be computed through gradient descent.
%
%\cite{Smith2009} showed the limitations of the mutual information as a privacy cost and introduced the conditional min-entropy which presents nice properties. It corresponds to the Reny (or $\alpha$) divergence with $\alpha \to \infty$. Thus, it seems interesting to extend our results for general $f$-divergences and not only limit them to the Kullback Leibler divergence. Also, it might of course been interesting to consider even more general privacy cost functions.
%
We formalized a new utility-privacy trade-off problem to compute strategies revealing private information only if it induces a significant increase in utility. For classical costs in economics, it benefits from recent advances of Optimal Transport. %, which make the problem efficiently solvable for complex action spaces or utility costs. 
It yet leads to a hard non-convex minimization problem and future work includes designing efficient algorithms computing global minima for this problem.

We believe that this work is a step towards the design of optimal utility vs. privacy trade-offs in economic mechanisms as well as for other applications. Its many connections with recent topics of interest motivate a better understanding of them as future work.

\addcontentsline{toc}{section}{References}
\bibliographystyle{abbrvnat}
\bibliography{private}

\clearpage
\appendix
\section{Proofs}
\label{app:proofs}

\begin{proof}[Proof of Theorem~\ref{thm:convexprob}]
The constraint set is obviously convex. The first integral is linear and thus convex in $\gamma$. It remains to show that the privacy loss is also convex in $\gamma$.
As $D$ is an $f$-divergence, the privacy cost is
\begin{equation*}
\begin{aligned}
c_{\text{priv}}(\gamma) & \coloneqq \int_{\mathcal{X} \times \mathcal{Y}} D\left( p_x, p_0\right)\mathrm{d}\gamma(x,y) \\ & = \int_{\mathcal{X}} \int_{\mathcal{Y}} f\Big(\frac{\mathrm{d}\gamma(x,y)}{\mathrm{d}\gamma_1(x) \mathrm{d}p_0(y)} \Big) \mathrm{d}p_0(y)\mathrm{d}\gamma_1(x),
\end{aligned}
\end{equation*}

where $\gamma_1 = \pi_{1 \#} \gamma$.
For $t \in (0,1)$ and two distributions $\gamma$ and $\mu$, we can define the convex combination $\nu = t \gamma + (1-t) \mu$. By linearity of the projection $\pi_1$, $\nu_1 = t \gamma_1 + (1-t) \mu_1$.
The convexity of $c_{\text{priv}}$ actually results from the convexity of the \textit{perspective} of $f$ defined by $g(x_1,x_2) = x_2 f(x_1/x_2)$ \citep{boyd2004}. It indeed implies
\begin{small}\begin{equation*} \begin{aligned}
t f\Big(\frac{\mathrm{d}\gamma(x,y)}{\mathrm{d}\gamma_1(x) \mathrm{d}p_0(y)}\Big)\mathrm{d}\gamma_1(x) +  (1-t)f\Big(\frac{\mathrm{d}\mu(x,y)}{\mathrm{d}\mu_1(x) \mathrm{d}p_0(y)} \Big) \mathrm{d}\mu_1(x) \\ \geq f\Big(\frac{\mathrm{d}\nu(x,y)}{\mathrm{d}\nu_1(x) \mathrm{d}p_0(y)} \Big) \mathrm{d}\nu_1(x). \end{aligned}
\end{equation*}\end{small}

The result then directly follows when summing over $\mathcal{X}\times \mathcal{Y}$.
\end{proof}
\medskip

\begin{proof}[Proof of Theorem~\ref{thm:caratheodory1}]
We first consider $\varepsilon > 0$. Let $(p^k)_k$ and $\mu$ be an $\varepsilon$-optimal solution of this problem. We define $g_0(x) = \sum_k \big[ p^k(x) c(x, y_k) + \lambda p_0^k(x) f\left( \frac{p^k(x)}{p_0^k}\right) \big]$ and $g_k = p^k$ for $1 \leq k \leq K$. Let $\alpha_j(\mu) = \int_\mathcal{X} g_j \mathrm{d}\mu$ for $0 \leq j \leq K$. The considered solution $\mu$ is included in a convex hull as follows:
\begin{equation*}
(\alpha_j(\mu))_{0 \leq j \leq K} \in \mathrm{Conv} \{ (g_j(x))_{0 \leq j \leq K} \ / \ x \in \mathcal{X} \}.
\end{equation*}

So by Caratheodory theorem, there exist $K+2$ points $x_i \in \mathcal{X}$ such that $(\alpha_j(\mu))_{0 \leq j \leq K} = \sum_{i=1}^{K+2} t_i (g_j(x_i))_{0 \leq j \leq K}$ where $(t_i)$ is a convex combination.
Let $\mu' = \sum_{i=1}^{K+2} t_i \delta_{x_i}$. We then have $\alpha_j(\mu') = \alpha_j(\mu)$ for all $j$, which means that $(\mu', (p^k)_k)$ is also an $\varepsilon$-optimal solution of the problem given at the beginning of Section~\ref{sec:finite} and the support of $\mu'$ is a subset of $\lbrace x_i \ | \ i \in [K+2] \rbrace$.

Let $h_{k}(\gamma_i) \coloneqq \left(\sum_{m=1}^K \gamma_{i,m}\right) f \left(\frac{\gamma_{i,k}}{p_0^k\sum_{m=1}^K \gamma_{i,m}}\right)$, with the conventions $f(0) = \lim\limits_{x\to 0} f(x) \in \mathbb{R}\cup\lbrace +\infty \rbrace$ and ${h_{k}(\gamma_i) = 0}$ if $\sum_{m=1}^K\gamma_{i,m}=0$.

The privacy cost is then the sum of the $h_{k}(\gamma_i)$ for all $k$ and $i$. The case $\varepsilon=0$ comes from the lower semi-continuity of the objective function, as claimed by Lemma~\ref{lemma:lsc}, whose proof is given below. 
\begin{lemma}
\label{lemma:lsc}
For any $k$ in $\lbrace 1 , \ldots, K \rbrace$, $h_{k}$ is lower semi-continuous.
\end{lemma}
Let $(\gamma^{(n)}, x^{(n)})_n$ be a feasible sequence whose value converges to this infimum. As the constraint set is compact, we can assume after extraction that $(x^{(n)}, \gamma^{(n)}) \to (x, \gamma)$. As $c(\cdot, y_k)$ and $h_{k}$ are all lower semi-continuous, the value of the limit is smaller than the limit value. This implies that the infimum is reached in $(\gamma, x)$.
\end{proof}
\medskip

\begin{proof}[Proof of Lemma~\ref{lemma:lsc}]
$f$ is convex on $\mathbb{R}_+^*$. As a consequence, it is also continuous on $\mathbb{R}_+^*$. If $\lim\limits_{x\to 0^+} f(x) \in \mathbb{R}$, then $f$ can be extended as a continuous function on $\mathbb{R}_+$ and all the $h_{k}$ are thus continuous. 

Otherwise, $f$ is still continuous on $\mathbb{R}_+^*$ and ${\lim\limits_{x\to 0^+} f(x) = + \infty}$. Thus, $h_{k}$ is continuous at $\gamma_i$ as soon as $\gamma_{i,j}>0$ for every $j$. If $\gamma_{i,k}=0$, but the sum $\sum_{l=1}^K\gamma_{i,l}$ is strictly positive, then $h_{k}(\gamma_i) = + \infty$, but as soon as $\rho \to \gamma$, we also have an infinite limit, because $\frac{\rho_{i,k}}{p_0^k\sum_l\rho_{i,l}} \to 0$.

If $\sum_{l=1}^K\gamma_{i,l}=0$, then $\liminf\limits_{\rho \to \gamma} f\Big(\frac{ \rho_{i,k}}{p_0^k\sum\limits_l\rho_{i,l}} \Big) \in \mathbb{R} \cup \lbrace + \infty \rbrace$. As we multiply this term by a factor going to $0$, $\liminf\limits_{\rho \to \gamma} h_{k}(\rho_i) \geq 0=h_{k}(\gamma_i)$. Thus, in any case, $h_{k}$ is lower semi-continuous.
\end{proof}

\medskip
\begin{proof}[Proof of Corollary~\ref{coro:eqproblem}]
Theorem~\ref{thm:caratheodory1} claims that \eqref{eq:staticmin1} is equivalent to the problem of Corollary~\ref{coro:eqproblem} if we also impose $x_i \neq x_j$ for $i \neq j$. The value of the latter is thus lower than the value of the former as we consider a larger feasible set.
Let us consider a redundant solution $(\gamma, x)$ with $x_i = x_j$ for $i \neq j$. We now show that a non redundant version of this solution will have a lower value.

The functions $h_{k}$ are convex functions because of the convexity of the perspectives of convex functions \citep{boyd2004}. Also, they are obviously homogeneous of degree $1$. These two properties imply that the $h_{k}$ are subadditive. Thus, let $(\gamma', x')$ be defined by $\gamma'_{l,k} \coloneqq \gamma_{l,k}$ and $x'_l \coloneqq x_l$ for any $l \not\in \lbrace i,j \rbrace$; $\gamma'_{i,k} \coloneqq \gamma_{i,k} + \gamma_{j,k}$, $\gamma'_{j,k} \coloneqq 0$ for any $k$, $x'_i \coloneqq x_i$ and let $x'_j$ be any element in $\mathcal{X} \setminus \lbrace x_l \ | \ 1 \leq l \leq K+2 \rbrace$. The subadditivity of $h_k$ implies $h_k (\gamma'_i) + h_k(\gamma'_j) \leq h_k (\gamma_i) + h_k(\gamma_j)$ for any $k$. The other terms in the objective function will be the same for $(\gamma, x)$ and $(\gamma', x')$. It thus holds $\sum\limits_{i,k} \gamma_{i,k} c(x_i, y_k) + \lambda \sum\limits_{i,k} p_0^k h_{k}(\gamma_i) \geq \sum\limits_{i,k} \gamma'_{i,k} c(x'_i, y_k) + \lambda \sum\limits_{i,k} p_0^k h_{k}(\gamma'_i)$.

\medskip

$(\gamma', x')$ is in the feasible set of the problem of Corollary~\ref{coro:eqproblem} and we removed a redundant condition from $x$. We can thus iteratively construct a solution $(\tilde{\gamma}, \tilde{x})$ until reaching non redundancy. We then have $(\tilde{\gamma}, \tilde{x})$ a non redundant solution and its value is lower than the value of $(\gamma, x)$. This means that allowing redundancy does not change the infimum of the problem, \ie this yields Corollary~\ref{coro:eqproblem}.
\end{proof}

\medskip

\begin{proof}[Proof of Theorem~\ref{thm:linearcost}]
Let $\psi$ be the rescaling of $\mathcal{X}$ to $[-1, 1]^d$, \ie $\psi(x)^l \coloneqq \frac{2x^l - b_l - a_l}{b_l - a_l}$. Then, $c(x,y) =  \psi(x)^T \phi(y) + \eta(y)$ where $\phi(y)^l \coloneqq (b_l - a_l) \frac{y^l}{2}$ and $\eta(y) = \sum\limits_{l=1}^d \frac{a_l + b_l}{b_l - a_l} y^l$.
The problem given by Corollary~\ref{coro:eqproblem} is then equivalent to:
\begin{small}
\begin{align*}
\inf_{\substack{(\gamma, x) \\ \gamma \in \mathbb{R}_+^{(K+2) \times K} \\ x \in [-1, 1]^{d \times (K+2)} }} & \sum\limits_{i,k} \gamma_{i,k} ( x_i^\top \phi(y_k) + \eta(y_k)) + \lambda \sum\limits_{i,k} p_0^k h_{k}(\gamma_i)\\
\text{such that } & \forall k \leq K, \sum_{i} \gamma_{i,k} = p_0^k .
\end{align*}
\end{small}

Because of the marginal constraints, $\sum\limits_{i,k} \gamma_{i,k} \eta(y_k) = \sum\limits_k p_0^k \eta(y_k)$. This sum does not depend neither on $x$ nor $\gamma$, so that the terms $\eta(y_k)$ can be omitted. The problem then becomes

\begin{small}
\begin{align*}
\inf_{\substack{(\gamma, x) \\ \gamma \in \mathbb{R}_+^{(K+2) \times K} \\ x \in [-1, 1]^{d \times (K+2)} }} & \sum\limits_{i} x_i^\top \left( \sum\limits_k \gamma_{i,k} \phi(y_k) \right)  + \lambda \sum\limits_{i,k} p_0^k h_{k}(\gamma_i)\\
\text{such that } & \forall k \leq K, \sum_{i} \gamma_{i,k} = p_0^k .
\end{align*}
\end{small}

It is clear that for a fixed $\gamma$, the best $x_i$ corresponds to $x_i^l = - \text{sign}(\sum\limits_k \gamma_{i,k} \phi(y_k)^l)$ and the term $x_i^\top \left( \sum\limits_k \gamma_{i,k} \phi(y_k) \right)$ then corresponds to the opposite of the $1$-norm of $\sum\limits_k \gamma_{i,k} \phi(y_k)$, \ie the problem is then:
\begin{align*}
\inf\limits_{\gamma \in \mathbb{R}_+^{(K+2) \times K}} & - \sum\limits_{i}  \| \sum\limits_k \gamma_{i,k} \phi(y_k) \|_1  + \lambda \sum\limits_{i,k} p_0^k h_{k}(\gamma_i)\\
\text{such that } & \forall k \leq K, \sum_{i} \gamma_{i,k} = p_0^k .
\end{align*}
\end{proof}

\end{document}